\documentclass[letterpaper, 10 pt, conference]{ieeeconf}  

\IEEEoverridecommandlockouts                              

\usepackage{comment}
\usepackage{xcolor}
\usepackage{amsmath}
\usepackage{amsfonts}
\usepackage{amsthm}
\usepackage[short,c2,nocomma]{optidef}
\usepackage{algorithm}
\usepackage[noEnd]{algpseudocodex}

\let\labelindent\relax
\usepackage{enumitem}

\usepackage[hang,flushmargin]{footmisc} 
\makeatletter
\renewcommand{\@makefnmark}{\hbox{\textsuperscript{\raisebox{-0.1ex}{\normalfont\@thefnmark}}}}
\makeatother

\newtheorem{lemma}{Lemma}{\bfseries}{\itshape}
\newtheorem{theorem}{Theorem}{\bfseries}{\itshape}
\theoremstyle{remark}
\newtheorem{remark}{Remark}

\def\ntar{{n_\textnormal{tar}}}
\def\nagt{{n_\textnormal{agt}}}
\def\taua{{\tau_\textnormal{a}}}
\def\LFDT{{\textnormal{LFDT}}}
\def\EFAT{{\textnormal{EFAT}}}
\def\Length{{\textnormal{Len}}}

\def\LPB{{\textnormal{LP-}\mathcal{B}}}
\def\LB{{\textnormal{LB}}}
\def\UB{{\textnormal{UB}}}

\def\dist{{\textnormal{dist}}}

\makeatletter
\newcommand\fs@betterruled{%
  \def\@fs@cfont{\bfseries}\let\@fs@capt\floatc@ruled
  \def\@fs@pre{\vspace*{5pt}\hrule height.8pt depth0pt \kern2pt}%
  \def\@fs@post{\kern2pt\hrule\relax}%
  \def\@fs@mid{\kern2pt\hrule\kern2pt}%
  \let\@fs@iftopcapt\iftrue}
\floatstyle{betterruled}
\restylefloat{algorithm}
\makeatother

\title{\LARGE \bf
Optimal Solutions for the Moving Target Vehicle Routing Problem with Obstacles via Lazy Branch-and-Price
}

\author{Anoop Bhat$^{1}$ and Geordan Gutow$^{2}$ and David Neiman$^{1}$ and Surya Singh$^{3}$ and\\Zhongqiang Ren$^{4}$ and Sivakumar Rathinam$^{5}$ and Howie Choset$^{1}$
\thanks{$^{1}$Robotics Institute at Carnegie Mellon University, 5000 Forbes Ave., Pittsburgh, PA 15213. Emails: \{agbhat,
choset\}@andrew.cmu.edu.}%
\thanks{$^{2}$Mechanical and Aerospace Engineering at Michigan Technological University, Houghton, MI 49931. Email: gmgutow@mtu.edu}%
\thanks{$^{3}$Robotics and AI Institute, Cambridge, MA 02142. Email: ssingh@rai-inst.com}%
\thanks{$^{4}$UM-SJTU Joint Institute and Department of Automation at Shanghai Jiao Tong University, Shanghai, China. Email: zhongqiang.ren@sjtu.edu.cn}%
\thanks{$^{5}$Department of Mechanical Engineering and Department of Computer Science and Engineering at Texas A\&M University, College Station, TX 77843. Email: srathinam@tamu.edu}%
}

\begin{document}

\maketitle
\thispagestyle{empty}
\pagestyle{empty}

\begin{abstract}
The Moving Target Vehicle Routing Problem with Obstacles (MT-VRP-O) seeks trajectories for several agents that collectively intercept a set of moving targets. Each target has one or more time windows where it can be visited, and the agents must avoid static obstacles and satisfy speed and capacity constraints. Previously studied VRPs are often addressed using a framework called branch-and-price. This framework requires computing the cost for a single agent to visit a given sequence of target-time window pairings, for several candidate sequences. These sequences are called tours. Computing tour costs is more expensive in the MT-VRP-O than in previously studied VRPs due to the presence of both moving targets and static obstacles. Thus, we introduce a new exact algorithm, Lazy Branch-and-Price with Relaxed Continuity (Lazy BPRC), for the MT-VRP-O. The key idea in Lazy BPRC is to use cheap-to-compute lower bounds on tour costs where costs are traditionally used, and lazily update lower bounds to true costs. When computing a tour’s true cost is needed, we do so by searching for a shortest path on a graph of convex sets, and we introduce a new heuristic to accelerate the search. We demonstrate that Lazy BPRC runs up to an order of magnitude faster than two ablations.
\end{abstract}

\section{INTRODUCTION}
Finding trajectories for multiple agents to visit multiple moving targets is necessary in applications such as defense \cite{helvig2003moving,smith2021assessment,stieber2022DealingWithTime}, orbital refueling \cite{bourjolly2006orbit}, and recharging mobile robots collecting data from the seafloor \cite{Li2019RendezvousPlanning}. These applications can be modeled as variations of the Vehicle Routing Problem (VRP) \cite{toth2014vehicle,archetti2025beyond}. The VRP assumes a set of stationary targets and a set of agents. The agents all start at a location called the depot. Each target has a demand of goods. Each agent has a capacity on the amount of goods it can deliver. Given the travel cost between every pair of targets, and between targets and the depot, the VRP seeks a sequence of targets for each agent with minimal sum of costs, such that the sum of demands of targets visited by an agent does not exceed the capacity. In the Moving Target VRP (MT-VRP) \cite{bhat2026optimal}, the targets are moving, and we seek not only a sequence of targets for each agent, but a trajectory intercepting each target in the sequence. Each target has one or more time windows where it can be met, and the agents have a speed limit. Prior work on the MT-VRP assumes piecewise-linear target trajectories \cite{bhat2026optimal}, and we do so as well. In this work, we address a generalization of the MT-VRP where the agents must avoid static obstacles. We call this generalization the MT-VRP with Obstacles (MT-VRP-O), illustrated in Fig. \ref{fig:intro_fig}.

\begin{figure}
    \centering
    \includegraphics[width=0.47\textwidth]{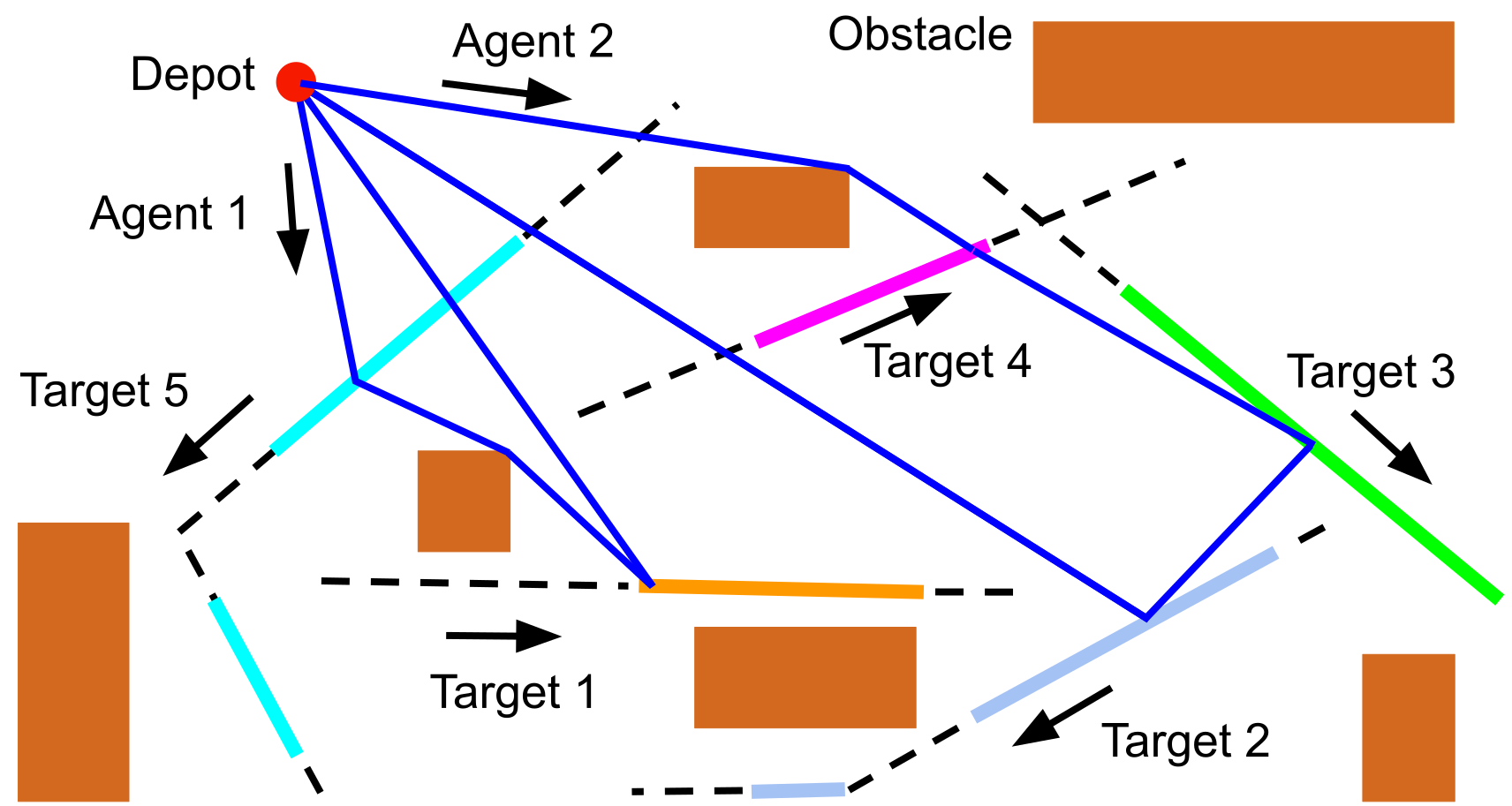}
    \vspace{-0.2cm}
    \caption{Targets move through obstacle environment. Time windows of targets are shown in bold colored lines, and each target must be intercepted in one of its time windows. Agents begin and end at depot, intercepting targets while avoiding obstacles.}
    \vspace{-0.5cm}
    \label{fig:intro_fig}
\end{figure}

The MT-VRP-O generalizes the Traveling Salesman Problem (TSP), and thus finding an optimal solution is NP-hard \cite{helvig2003moving}. No prior methods find an optimal solution for the MT-VRP-O. The closest related work finds optimal solutions for the MT-VRP without obstacles \cite{bhat2026optimal}, using the branch-and-price framework \cite{costa2019exact}. In this work, we develop a new branch-and-price algorithm for the MT-VRP-O called Lazy Branch-and-Price with Relaxed Continuity (Lazy BPRC).

Branch-and-price alternates between a restricted master problem (RMP) and a pricing problem. The RMP aims to select a sequence of target-time window pairings (called a tour) for each agent to follow, from a restricted subset of tours. The pricing problem adds tours to the restricted subset. Conventionally, solving the RMP requires computing the cost for an agent to follow each tour in the restricted subset. The challenge in the MT-VRP-O is that computing a tour's cost requires collision-free motion planning between moving targets. Thus, Lazy BPRC defers computing tour costs by solving the RMP using lower bounds on costs, computed via motion planning with relaxed continuity constraints. We lazily compute true costs as needed, by searching for a shortest path on a Graph of Convex Sets (GCS). We use our continuity relaxation strategy in a heuristic for the search. Our results show that Lazy BPRC runs up to 43 times faster than a non-lazy ablation, and up to 23 times faster than an ablation using an existing obstacle-unaware heuristic \cite{bhat2025AComplete}.

\section{RELATED WORK}
While prior work has not studied the MT-VRP-O, related problems have been studied. For example, \cite{Li2019RendezvousPlanning} studies a multi-agent Moving Target TSP with Obstacles (multi-agent MT-TSP-O), which lacks the capacity constraints from the MT-VRP-O. However, their approach only allows interception at sampled points along the targets' trajectories, and thus \cite{Li2019RendezvousPlanning} does not provide optimal solutions, unlike our method. On the other hand, for the single-agent MT-TSP-O, \cite{bhat2025AComplete} presents a solver that finds optimal solutions. This approach alternates between a high-level search to generate a tour, and a low-level search to find a trajectory intercepting the tour's targets to compute its cost. The low-level search in \cite{bhat2025AComplete} solves a Shortest Path Problem on a GCS (SPP-GCS). We similarly solve an SPP-GCS to evaluate a tour's cost, and we provide a novel heuristic for the search that we show outperforms the heuristic from \cite{bhat2025AComplete}.

The work in \cite{bhat2026optimal} addresses the MT-VRP without obstacles, using an approach called Branch-and-Price with Relaxed Continuity (BPRC). Our approach, Lazy BPRC, extends BPRC to handle obstacles, using a new obstacle-aware continuity relaxation strategy, as well as lazy tour cost evaluation. We show in Section \ref{sec:numerical_results} that our lazy evaluation outperforms BPRC's non-lazy tour cost evaluation.

\section{PROBLEM SETUP}\label{sec:problem_setup}
We have $n_\text{tar}$ targets moving in $\mathbb{R}^2$. The set of targets is $\{1, 2, \dots, n_\text{tar}\}$. Each target $i$ has a demand of goods $d_i$. Target $i$ has $n_{\text{win}}(i)$ time windows, and $[\underline{t}_{i,j}, \overline{t}_{i,j}]$ is the $j$th time window of target $i$. The trajectory of target $i$ is $\tau_i : \mathbb{R} \rightarrow \mathbb{R}^2$. We assume $\tau_i$ has constant velocity in each time window, but possibly different velocities in different time windows. We assume targets do not pass through obstacles during their time windows.

The number of agents is $n_\text{agt}$. Each agent has a capacity $d_\text{max}$ on the amount of goods it can deliver. When visiting a target, an agent must deliver the target's full demand. Each agent has a speed limit $v_\text{max}$. No target moves faster than $v_\text{max}$ in its time windows. We denote an agent's trajectory as $\taua$. An agent trajectory $\taua$ \emph{intercepts} target $i$ if $\taua(t) = \tau_i(t)$ at some $t$ in some time window of target $i$. All agents start at a depot $p_\text{d} \in \mathbb{R}^2$. The agents must avoid collision with static obstacles, i.e. they must never enter the interior of any obstacle.\footnote{If two obstacles intersect at exactly one point, we also require that no agent passes through this point.} We call a collision-free agent trajectory satisfying the speed limit a \emph{feasible} agent trajectory.

The MT-VRP-O seeks an assignment of targets to agents, and a feasible trajectory for each agent intercepting its assigned targets, such that every target is assigned to some agent, and for each agent, the sum of demands of its assigned targets does not exceed its capacity. In this work, we aim to minimize the sum of the agents' distances traveled.

\section{INTEGER LINEAR PROGRAM (ILP) FOR MT-VRP-O}\label{sec:ilp_mt_vrp_o}
We consider a \emph{target-window graph} $\mathcal{G}_\text{tw} = (\mathcal{V}_\text{tw}, \mathcal{E}_\text{tw})$. Each node in $\mathcal{V}_\text{tw}$ is a pairing of a target $i$ with one of its time windows, called a \emph{target-window}. For example, $\gamma_{i,j} = (i, [\underline{t}_{i,j}, \overline{t}_{i,j}])$ denotes the $j$th target-window of target $i$. $\mathcal{V}_\text{tw}$ contains all possible target-windows and two fictitious target-windows corresponding to the depot: $\gamma_{0, 1} = (0, [0, 0])$, and $\gamma_{0, 2} = (0, [0, \infty))$. Here, the depot is treated as a fictitious target 0. An agent trajectory $\taua$ \emph{intercepts} target-window $\gamma_{i,j}$ if $\taua$ intercepts target $i$ at some $t \in [\underline{t}_{i,j}, \overline{t}_{i,j}]$. $\mathcal{E}_\text{tw}$ has an edge from $\gamma_{i,j}$ to $\gamma_{i',j'}$ if $i \neq i'$. A path in $\mathcal{G}_\text{tw}$ \emph{visits} a target if it contains (i.e. visits) one of the target's target-windows.

A \emph{tour} is a path in $\mathcal{G}_\text{tw}$ beginning at $\gamma_{0, 1}$, ending at $\gamma_{0, 2}$, and visiting each non-fictitious target at most once, such that (i) the sum of demands of visited targets is at most $d_\text{max}$, and (ii) a feasible agent trajectory exists intercepting the tour's target-windows in sequence. For a tour $\Gamma$, let $\Gamma[n]$ denote the $n$th element of $\Gamma$; we use the same bracket notation to indicate the $n$th element of any sequence. Let $\Length(\Gamma)$ denote the number of target-windows in $\Gamma$. An agent trajectory $\taua$ \emph{follows} $\Gamma$ if $\taua$ intercepts the target-windows in $\Gamma$ in sequence, and $\taua$ is feasible. For a tour $\Gamma$, the cost of $\Gamma$, denoted as $c^*(\Gamma)$, is the distance traveled by a minimum-distance trajectory following $\Gamma$. We compute the cost of a tour by solving an SPP-GCS, described in Section \ref{sec:tour_cost}.

Let the set of all tours be $\mathcal{S}$. We formulate the MT-VRP-O as the problem of selecting a set $\mathcal{F}_\text{sol} \subseteq \mathcal{S}$, containing up to $\nagt$ tours, such that every target is visited by some selected tour, and the sum of tour costs is minimized. In particular, for a tour $\Gamma$, let $\alpha(i, \Gamma) = 1$ if $\Gamma$ visits target $i$ and let $\alpha(i, \Gamma) = 0$ otherwise. Define a binary variable $\theta_k$ which equals 1 if tour $\Gamma_k$ is selected and 0 otherwise. We formulate the MT-VRP-O as the following ILP:
\begin{mini!}
{\{\theta_k\}_{\Gamma_k \in \mathcal{S}}}{\sum\limits_{\Gamma_k \in \mathcal{S}}c^*(\Gamma_k)\theta_k\label{eqn:ilp_objective}}{\label{optprob:mt_vrp_o_ilp}}{}
\addConstraint{\sum\limits_{\Gamma_k \in \mathcal{S}}\theta_k \leq \nagt\label{eqn:ilp_agent_limit}}{}
\addConstraint{\sum\limits_{\Gamma_k \in \mathcal{S}}\alpha(i, \Gamma_k)\theta_k\geq 1 \;\; \forall i \in \{1, \dots, \ntar\}\label{eqn:ilp_visit_all_targets}}{}
\addConstraint{\theta_k \in \{0, 1\} \;\; \forall \Gamma_k \in \mathcal{S}\label{eqn:ilp_theta_binary}}{}
\end{mini!}
\eqref{eqn:ilp_objective} minimizes the sum of tour costs, \eqref{eqn:ilp_agent_limit} ensures that no more than $\nagt$ tours are selected, \eqref{eqn:ilp_visit_all_targets} ensures all targets are visited, and \eqref{eqn:ilp_theta_binary} enforces that each $\theta_k$ is binary.
\section{LAZY BPRC ALGORITHM}\label{sec:lazy_bprc}
\subsection{Overview}\label{sec:overview}
ILP \eqref{optprob:mt_vrp_o_ilp} poses three challenges:
\begin{enumerate}
    \item The binary constraint \eqref{eqn:ilp_theta_binary} is nonconvex.
    \item The number of tours, and thus the number of variables, is factorial in the number of targets.
    \item Computing the tour costs in the objective \eqref{eqn:ilp_objective} requires collision-free motion planning.
\end{enumerate}
To address the first challenge, we use a convex relaxation of ILP \eqref{optprob:mt_vrp_o_ilp}, which replaces the binary constraint \eqref{eqn:ilp_theta_binary} with $\theta_k \geq 0$. This relaxation allows \emph{fractional solutions}, where some $\theta_k$ values take non-integer values. Rather than strictly selecting or not selecting a tour, a fractional solution can select a tour $\Gamma_k$ with some proportion, e.g. by setting $\theta_k = 0.5$. Thus, we perform a branch-and-bound search, which seeks a set of constraints to add to the relaxation such that the optimal solution for ILP \eqref{optprob:mt_vrp_o_ilp} is optimal for the constrained relaxation.

This search requires solving several constrained relaxations, which still pose the latter two challenges above: a large number of variables and expensive-to-compute tour costs. To handle the large number of variables, we use a standard method called \emph{column generation} \cite{feillet2010tutorial}. Column generation maintains a \emph{restricted subset} of tours, $\mathcal{F}$. When solving a constrained relaxation, we only optimize over $\theta_k$ for $\Gamma_k \in \mathcal{F}$, implicitly fixing other $\theta_k$ to zero, and lazily add tours to $\mathcal{F}$ while optimizing. Column generation makes Lazy BPRC a branch-and-price algorithm rather than only branch-and-bound.

Our key contribution is how we address expensive-to-compute tour costs. In particular, rather than computing the cost of a tour $\Gamma_k$ when we add $\Gamma_k$ to $\mathcal{F}$, we compute lower and upper bounds on the cost, denoted as $\LB(\Gamma_k)$ and $\UB(\Gamma_k)$, respectively. Fig. \ref{fig:bounding_tour_costs} illustrates the bounds, and subsequent subsections give computation details. We use $\LB(\Gamma_k)$ and $\UB(\Gamma_k)$ where branch-and-price traditionally uses true costs, and we lazily update bounds to equal costs.

Section \ref{sec:branch_and_bound} describes the branch-and-bound search. Sections \ref{sec:lower_bound_on_tour_cost} and \ref{sec:upper_bound_on_tour_cost} describe how we compute lower and upper bounds on tour costs. Section \ref{sec:tour_cost} describes how we compute tour costs.
\begin{figure}
    \centering
    \vspace{0.2cm}
    \includegraphics[width=0.47\textwidth]{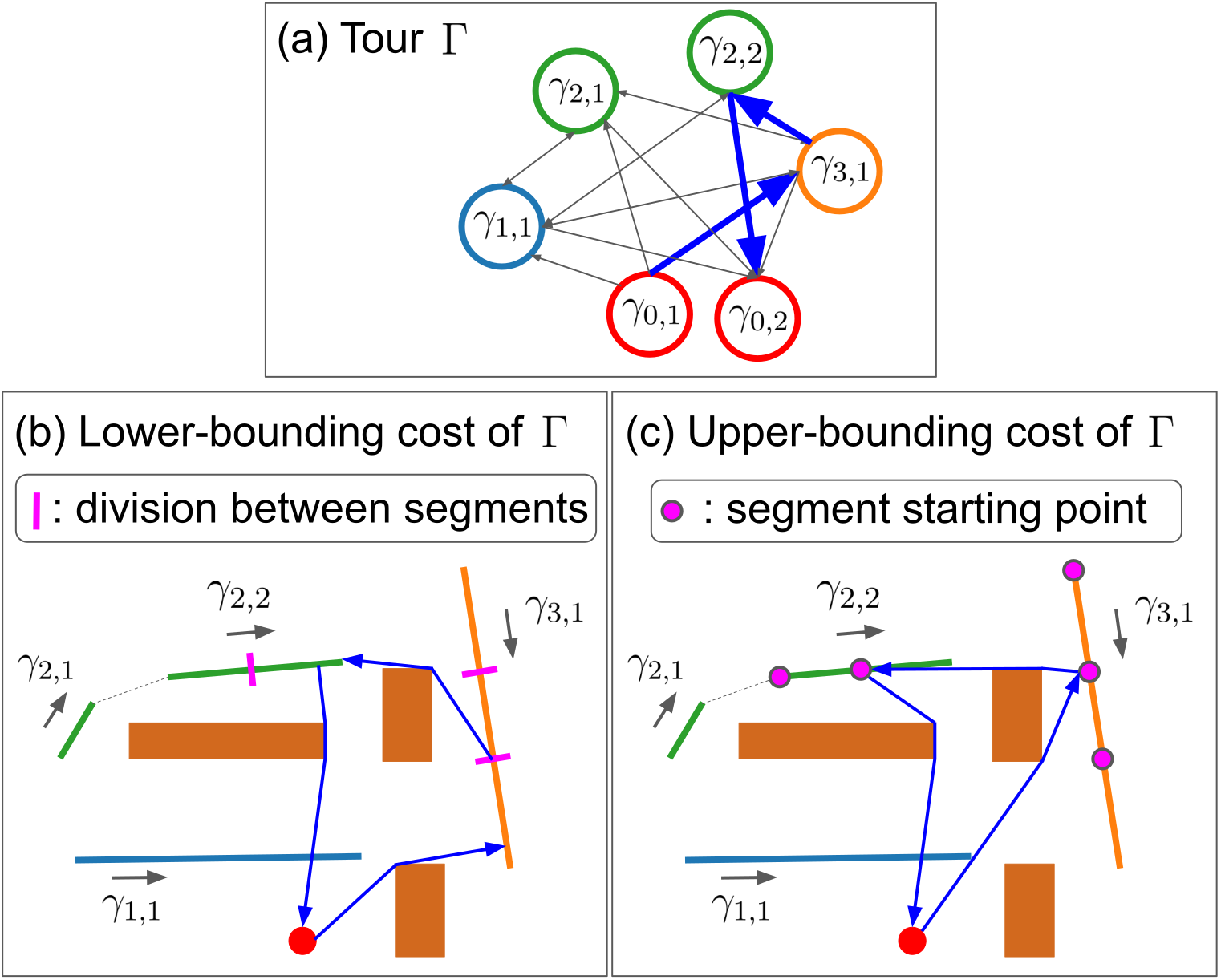}
    \vspace{-0.3cm}
    \caption{(a) Example tour $\Gamma$. (b) To compute the lower bound $\LB(\Gamma)$ on $c^*(\Gamma)$, we divide each target-window visited by $\Gamma$ into segments. A lower bound on $c^*(\Gamma)$ is the cost of a trajectory $\taua$ that optimally follows $\Gamma$ subject to a relaxed continuity constraint: the arrival and departure points at a target may differ, but must lie on the same segment. $\LB(\Gamma)$ lower-bounds the cost of such a discontinuous trajectory. (c) The upper bound $\UB(\Gamma)$ is the cost of a trajectory that optimally follows $\Gamma$ subject to the constraint that targets are only intercepted at segment start points.}
    \vspace{-0.7cm}
    \label{fig:bounding_tour_costs}
\end{figure}

\subsection{Branch-and-Bound}\label{sec:branch_and_bound}
Alg. \ref{alg:LazyBPRC} describes the branch-and-bound. Line \ref{algline:BPRC_initialize_stack} initializes a stack of \emph{branch-and-bound nodes}, where each node represents a set of constraints to add to the relaxation of ILP \eqref{optprob:mt_vrp_o_ilp}, to eliminate fractional solutions. As in \cite{ozbaygin2017branch}, we only use constraints that disallow any $\Gamma_k$ traversing a particular edge $e \in \mathcal{E}_\text{tw}$, i.e. constraints that fix $\theta_k = 0$ for such $\Gamma_k$. Thus, a branch-and-bound node is a set of disallowed edges $\mathcal{B} \subseteq \mathcal{E}_\text{tw}$.

\begin{algorithm}
\caption{Lazy BPRC}\label{alg:LazyBPRC}
\begin{algorithmic}[1]
\State STACK = [$\emptyset$]\label{algline:BPRC_initialize_stack}
\State $\mathcal{F}_\text{inc} =$ GenerateFeasibleSolution()\label{algline:bprc_init_incumbent}
\State $\mathcal{F} = \mathcal{F}_\text{inc}$\label{algline:init_restricted_subset}
\While{STACK not empty}\label{algline:BPRC_main_loop}
    \State $\mathcal{B}$ = STACK.pop()
    \While{true}\Comment{\textnormal{Solve}-$\LPB$ loop}\label{algline:solve_LPB}
        \State $\theta, \lambda$ = SolveRMP($\mathcal{F}$)\label{algline:solve_rmp}
        \State UpdateIncumbent($\theta, \mathcal{F}_\text{inc}$)\label{algline:update_incumbent_after_rmp}
        \State added = AddPotentiallyImprovingTours($\lambda, \mathcal{F}$)\label{algline:find_improving_tours}
        \If{added} continue
        \EndIf
        \If{$\theta$ is integer and $\LB(\theta) < \UB_\text{inc}$}\label{algline:check_eval}
            \State EvaluateTours($\theta$)\label{algline:eval_tours}
            \State UpdateIncumbent($\theta, \mathcal{F}_\text{inc}$)\label{algline:update_incumbent_after_eval}
        \Else{} break
        \EndIf
    \EndWhile
    \If{$\LB(\theta) \geq \UB_\text{inc}$}\label{algline:BPRC_check_lb_worse_than_ub}
        continue\label{algline:BPRC_no_successors}
    \EndIf
    \State PushSuccessors(STACK, $\mathcal{B}, \theta$)
\EndWhile
\State return $\mathcal{F}_\text{inc}$
\end{algorithmic}
\end{algorithm}

Line \ref{algline:bprc_init_incumbent} generates an initial feasible solution to ILP \eqref{optprob:mt_vrp_o_ilp}, i.e. a set of tours $\mathcal{F}_\text{inc}$, using the method in Appendix A. We call $\mathcal{F}_\text{inc}$ the \emph{incumbent}. We continually update $\mathcal{F}_\text{inc}$ to be the solution with best sum of upper bounds found so far. Let $\UB_\text{inc}$ be the sum of upper bounds of tours in $\mathcal{F}_\text{inc}$. Line \ref{algline:init_restricted_subset} initializes the restricted subset of tours $\mathcal{F}$.

Each loop iteration begins by popping a branch-and-bound node $\mathcal{B}$ from the stack. Next, we enter the inner loop on Line \ref{algline:solve_LPB}. This loop solves the relaxation of ILP \eqref{optprob:mt_vrp_o_ilp}, with the constraint that $\theta_k = 0$ for any $\Gamma_k$ traversing an edge in $\mathcal{B}$. We call this constrained relaxation $\LPB$, and the loop is called the Solve-$\LPB$ loop. To defer tour cost computations, we define the \emph{master problem} as $\LPB$, with $c^*(\Gamma_k)$ replaced by $\LB(\Gamma_k)$ in the objective. To avoid solving for all variables, we define the \emph{restricted master problem} (RMP) as the master problem, but $\mathcal{S}$ is replaced with $\mathcal{F}$, i.e. the RMP only allows selecting tours in $\mathcal{F}$.

Each iteration of the Solve-$\LPB$ loop first solves the RMP. This yields a primal solution $\theta$ and a dual solution $\lambda$. $\lambda$ has the form $\lambda = (\lambda_0, \lambda_1, \dots, \lambda_\ntar)$, where $\lambda_0 \in \mathbb{R}_{\leq 0}$ is the dual variable for \eqref{eqn:ilp_agent_limit}, and for $i > 0$, $\lambda_i \in \mathbb{R}_{\geq 0}$ is the dual variable for \eqref{eqn:ilp_visit_all_targets}. The RMP could be infeasible at this step, which occurs if all feasible MT-VRP-O solutions that can be constructed from tours in $\mathcal{F}$ traverse an edge in $\mathcal{B}$. If the RMP is infeasible, we treat $\theta$ and $\lambda$ as NULL. If the RMP is feasible, and $\theta$ is integer, $\theta$ corresponds to a feasible solution $\mathcal{F}_\text{sol}$ to the MT-VRP-O. If the sum of upper bounds of tours in $\mathcal{F}_\text{sol}$ is smaller than $\UB_\text{inc}$, we set $\mathcal{F}_\text{inc} = \mathcal{F}_\text{sol}$ (Line \ref{algline:update_incumbent_after_rmp}). 

Next, we search for tours to add to $\mathcal{F}$ that could improve our RMP solution (Line \ref{algline:find_improving_tours}). If the RMP is infeasible, we use the method from Appendix A to search for a set of tours $\mathcal{F}_\text{new}$ feasible for the MT-VRP-O traversing no edge in $\mathcal{B}$. If we find $\mathcal{F}_\text{new}$, we add its tours to $\mathcal{F}$. If the RMP is feasible, we instead solve a \emph{pricing problem}, which seeks tours to add to $\mathcal{F}$ that lower the RMP's optimal cost. We solve the pricing problem by adapting the labeling algorithm from \cite{bhat2026optimal} to handle obstacles, as detailed in Appendix B. Whether or not the RMP is feasible, if we add tours to $\mathcal{F}$, we go back to Line \ref{algline:solve_rmp} and solve the RMP again.

Otherwise, let $\LB(\theta) = \sum\limits_{\Gamma_k \in \mathcal{F}}\LB(\Gamma_k)\theta_k$.\footnote{If the RMP is infeasible, let $\LB(\theta) = \infty$.} Line \ref{algline:check_eval} checks if $\theta$ is integer, i.e. if $\theta$ selects tours as opposed to fractions of tours, and if $\LB(\theta) < \UB_\text{inc}$. If both conditions hold, $\theta$ selects a set of tours that may be better than the incumbent. Thus, on Line \ref{algline:eval_tours}, we get the selected set of tours $\mathcal{F}_\text{sol}$. For each tour $\Gamma_k \in \mathcal{F}_\text{sol}$ whose cost has not been computed, we compute $c^*(\Gamma_k)$ and update $\LB(\Gamma_k)$ and $\UB(\Gamma_k)$ to equal $c^*(\Gamma_k)$. We refer to this update as \emph{evaluating} $\Gamma_k$. If the cost of $\mathcal{F}_\text{sol}$ is less than $\UB_\text{inc}$, we set $\mathcal{F}_\text{inc} = \mathcal{F}_\text{sol}$ (Line \ref{algline:update_incumbent_after_eval}). 

If we evaluated tours, we go back to Line \ref{algline:solve_rmp} and solve the RMP again. This is needed since evaluating tours changes the RMP's objective. Thus, the optimal solution $\theta^*$ to ILP \eqref{optprob:mt_vrp_o_ilp} could now be optimal for the RMP, and we solve the RMP again to ensure we find $\theta^*$. Otherwise, we break from the Solve-$\LPB$ loop. At this point, our RMP solution $\theta$ is not necessarily optimal for $\LPB$.
For example, some tour $\Gamma_k$ selected by $\theta$ may be unevaluated, with a cost much higher than its lower bound. Thus, the optimal solution to $\LPB$, which uses actual costs rather than lower bounds, may not select $\Gamma_k$.
However, we ensure $\LB(\theta)$ lower-bounds the optimal cost of $\LPB$ (Lemma \ref{lemma:final_rmp_cost_lb_mp_cost}).

Line \ref{algline:BPRC_check_lb_worse_than_ub} checks if $\LB(\theta) \geq \UB_\text{inc}$. If this holds, adding constraints to $\LPB$ cannot improve $\mathcal{F}_\text{inc}$, so we prune $\mathcal{B}$. However, if $\LB(\theta) < \UB_\text{inc}$, the failure of Line \ref{algline:check_eval} implies $\theta$ is fractional. We create two successors of $\mathcal{B}$, called $\mathcal{B}'$ and $\mathcal{B}''$, such that $\theta$ is feasible for neither $\LPB'$ nor $\LPB''$, but all integer solutions to $\LPB$ are feasible for $\LPB'$ or $\LPB''$. To do so, we first select an edge $e \in \mathcal{E}_\text{tw}$, not incident to the depot, that is neither disallowed nor mandated in $\mathcal{B}$, via ``conventional branching" \cite{ozbaygin2017branch}. We let $\mathcal{B}' = \mathcal{B} \cup \{e\}$. We define $\mathcal{B}''$ such that $e$ must be traversed in $\LPB''$, by disallowing other edges appropriately (see \cite{ozbaygin2017branch}). We then push $\mathcal{B}'$ and $\mathcal{B}''$ onto the stack.
 
\subsection{Computing Lower Bound on Tour Cost}\label{sec:lower_bound_on_tour_cost}
We now describe how we compute $\LB(\Gamma)$ for a tour $\Gamma$, as illustrated in Fig. \ref{fig:bounding_tour_costs} (b). We perform three steps. First, we divide the targets' trajectories into segments. Second, we compute costs of agent trajectories between every pair of segments at different targets, for pairs where such a trajectory exists. Third, we compute $\LB(\Gamma)$ as the cost of a concatenation of these pairwise trajectories. This concatenation would produce a possibly discontinuous trajectory, as in Fig. \ref{fig:bounding_tour_costs} (b).

\subsubsection{Segments}
We divide each target-window $\gamma_{i,j}$ into \emph{segments}, where $\xi_{i,j,k} = (i, [\underline{t}_{i,j,k}, \overline{t}_{i,j,k}])$ is the $k$th segment of $\gamma_{i,j}$. To determine the number of segments per target-window, we first specify the number of segments per target, denoted as $n_\text{seg,tar}$. For each target $i$, we allocate segments to its windows using $n_\text{seg,tar}$ within the formula from BPRC \cite{bhat2026optimal}, which gives more segments to longer windows. The depot gets two segments: $\xi_{0, 1} = \xi_{0, 1, 1}$, corresponding to $\gamma_{0, 1}$, and $\xi_{0, 2} = \xi_{0, 2, 1}$, corresponding to $\gamma_{0, 2}$. For a target-window $\gamma$, $n_\text{seg}(\gamma)$ is the number of segments allocated to $\gamma$.

\subsubsection{Computing Pairwise Trajectory Costs}
For every pair of segments $(\xi, \xi')$ at different targets, we compute the shortest collision-free path in space from $\xi$ to $\xi'$, ignoring time constraints, via the method from \cite{asano1986visibility}.\footnote{\cite{asano1986visibility} only returns the shortest path if it contains obstacle vertices. Thus, we first compute the shortest path between segments ignoring obstacles; if it is collision-free, we return this path. Otherwise, we run \cite{asano1986visibility}. Also, \cite{asano1986visibility} computes visibility polygons as a subroutine; we use CGAL \cite{cgal:eb-24a} for this, not the method in \cite{asano1986visibility}. Finally, to prevent agents from passing through single-point intersections between obstacles (discussed in Section \ref{sec:problem_setup}), our implementation inflates obstacles by 1e-7 when constructing the visibility graph for \cite{asano1986visibility} and visibility graphs used later in our algorithm.} Let $c$ be the path's distance traveled. Let $t$ be the start time of $\xi$, and let $t'$ be the end time of $\xi'$. Let $i$ be the target of $\xi$, and let $i'$ be the target of $\xi'$. Next, we check if $t + c/v_\text{max} \leq t'$. This checks if there exists an agent trajectory following the path, starting when target $i$ enters $\xi$, and finishing before target $i'$ leaves $\xi'$. This is a looser condition than requiring interception of $i$ and $i'$. It only ensures that when the agent starts following the path, the agent is on the same segment as target $i$, and after following the path and waiting until time $t'$, the agent is on the same segment as target $i'$. If the trajectory existence check passes, the corresponding trajectory has cost $c$.

\subsubsection{Computing Lower Bound}
Given a tour $\Gamma$, let $t_n$ be the earliest time at which $\Gamma[n]$ can be intercepted while following $\Gamma$. This is computed as in \cite{bhat2024AComplete}, Appendix A. We define a \emph{segment-graph} $\mathcal{G}_\text{seg} = (\mathcal{V}_\text{seg}, \mathcal{E}_\text{seg})$, where the set of nodes $\mathcal{V}_\text{seg}$ is the set of all segments of each $\Gamma[n]$. For each edge $(\Gamma[n], \Gamma[n + 1]) \in \mathcal{E}_\text{tw}$ traversed by $\Gamma$, we connect edges in $\mathcal{E}_\text{seg}$ from every segment of $\Gamma[n]$ to every segment of $\Gamma[n + 1]$ whose end time is at least the earliest interception time $t_{n + 1}$. For each edge $(\xi, \xi') \in \mathcal{E}_\text{seg}$, if the trajectory existence check passed for this edge in step 2, the cost is the trajectory cost. Otherwise, the cost is $\infty$.

$\LB(\Gamma)$ is the cost of a least-cost path in $\mathcal{G}_\text{seg}$ from $\xi_{0,1}$ to $\xi_{0,2}$. If $\Gamma$ is produced in the pricing problem, $\LB(\Gamma)$ is computed as a byproduct (Appendix B). If $\Gamma$ is produced by the initial feasible solution method (Appendix A), we note that $\mathcal{G}_\text{seg}$ is a directed acyclic graph (DAG), and we compute $\LB(\Gamma)$ using a DAG shortest path algorithm \cite{cormen2001introduction}. Concatenating the trajectories corresponding to the shortest path's edges would yield a trajectory as in Fig. \ref{fig:bounding_tour_costs} (b).

\subsection{Computing Upper Bound on Tour Cost}\label{sec:upper_bound_on_tour_cost}
We now describe how we compute $\UB(\Gamma)$ for a tour $\Gamma$. As shown in Fig. \ref{fig:bounding_tour_costs} (c), $\UB(\Gamma)$ is the cost of a trajectory that follows $\Gamma$ and only intercepts targets at sampled points in space-time. The sample points are the segment start points of the targets and two points at the depot: $(p_\text{d}, 0)$ and $(p_\text{d}, t_\text{max})$, where $t_\text{max}$ is an upper bound on the duration of a distance-optimal trajectory following $\Gamma$ (e.g. the latest ending time over all time windows, multiplied by 2).\footnote{We sample at segment start points because the labeling algorithm in Appendix B needs an upper bound on the cost of reaching each segment start point when following $\Gamma$.} We compute $\UB(\Gamma)$ in two steps. First, we compute costs of trajectories between every pair of sample points at different targets, for pairs where such a trajectory exists. Second, we compute $\UB(\Gamma)$ as the cost of a concatenation of these pairwise trajectories.

\subsubsection{Computing Pairwise Trajectory Costs} For every pair of sample points $s$ and $s'$, we compute the shortest path in space from $s$ to $s'$, ignoring timing, using Dijkstra's algorithm on a visibility graph \cite{lozano1979algorithm}. Let $t$ and $t'$ be the times of $s$ and $s'$, respectively, and let $c$ be the path's distance traveled. If $t + c/v_\text{max} \leq t'$, a feasible agent trajectory exists from $s$ to $s'$ and its cost is $c$.

\subsubsection{Computing Upper Bound} We construct a \emph{sample-point-graph} $\mathcal{G}_\text{samp} = (\mathcal{V}_\text{samp}, \mathcal{E}_\text{samp})$. The set of nodes $\mathcal{V}_\text{samp}$ is the set of all samples whose target-windows are visited by $\Gamma$. For each edge $(\gamma, \gamma') \in \mathcal{E}_\text{tw}$ traversed by $\Gamma$, we connect an edge in $\mathcal{E}_\text{samp}$ from every sample at $\gamma$ to every sample at $\gamma'$. The edge cost from $s$ to $s'$ is the trajectory cost from $s$ to $s'$ computed in step 2, if the trajectory existence check passed, and $\infty$ otherwise. 

$\UB(\Gamma)$ is the cost of a least-cost path in $\mathcal{G}_\text{samp}$ from $(p_\text{d}, 0)$ to $(p_\text{d}, t_\text{max})$. If $\Gamma$ is produced in the pricing problem, $\UB(\Gamma)$ is computed as a byproduct (Appendix B). If $\Gamma$ is produced by the initial feasible solution method (Appendix A), we note that $\mathcal{G}_\text{samp}$ is a DAG, and we compute $\UB(\Gamma)$ via a DAG shortest path algorithm \cite{cormen2001introduction}. Concatenating the trajectories corresponding to the shortest path's edges would yield a trajectory as in Fig. \ref{fig:bounding_tour_costs} (c).

\subsection{Computing Tour Cost via SPP-GCS}\label{sec:tour_cost}
We now describe how we compute the cost of a tour $\Gamma_k$, i.e. the objective coefficient of $\theta_k$ in ILP \eqref{optprob:mt_vrp_o_ilp}. This requires computing a trajectory that follows $\Gamma_k$.

At the beginning of Lazy BPRC, we decompose free space into convex regions $\mathcal{A}_1, \mathcal{A}_2, \dots, \mathcal{A}_{n_\text{reg}}$, where $n_\text{reg}$ is the number of regions. The purpose of the regions is to enforce collision-avoidance, by requiring our trajectory to stay in the union of the regions. Next, we define a GCS $\mathcal{G}_\text{cs} = (\mathcal{V}_\text{cs}, \mathcal{E}_\text{cs})$, where the set of nodes $\mathcal{V}_\text{cs}$ consists of convex sets in space-time. For each region $\mathcal{A}$ in our free space decomposition, we have a \emph{region-node} $\mathcal{X}_{\mathcal{A}} = \mathcal{A} \times \mathbb{R}$. For each target-window $\gamma_{i,j}$ visited by $\Gamma$, we have a \emph{window-node} $\mathcal{X}_{i,j}$, consisting of the set of space-time points along $\tau_i$ within $[\underline{t}_{i,j}, \overline{t}_{i,j}]$. Since we assumed that a target's velocity is constant within a time window, $\mathcal{X}_{i,j}$ is a line segment in space-time, which is a convex set. We refer to nodes in $\mathcal{G}_\text{cs}$ as \emph{GCS-nodes}. An edge connects a set $\mathcal{X}$ to a set $\mathcal{X}'$ if $\mathcal{X}$ intersects $\mathcal{X}'$, unless $\mathcal{X}$ and $\mathcal{X}'$ are both region-nodes whose regions intersect at exactly one point.\footnote{We do not connect an edge in the single-point intersection case to prevent the scenario discussed in Section \ref{sec:problem_setup}, where an agent passes through a single-point intersection between two obstacles.} We refer to a path in $\mathcal{G}_\text{cs}$ as a \emph{GCS-path}. We say a GCS-path $P$ \emph{visits} target-window $\gamma_{i,j}$ if $P$ contains $\mathcal{X}_{i,j}$. A trajectory $\taua$ \emph{follows} GCS-path $P$ if it passes through the sets in $P$ in sequence.

We find a trajectory following $\Gamma$ by finding a GCS-path $P$ visiting the target-windows in $\Gamma$ in sequence, as well as a trajectory following $P$. We use an algorithm similar to FMC* \cite{bhat2025AComplete}, but without the speedup methods specific to a minimum-time objective, since we minimize distance traveled. We also replace the heuristic in FMC* with a heuristic described later in this section. Finally, we replace the focal search in FMC* with best-first search, since we seek optimal solutions rather than bounded-suboptimal ones.

We search with a priority queue called OPEN, containing GCS-paths. We initialize OPEN with GCS-path $(\mathcal{X}_{0,1})$, which stays at the depot. Each GCS-path $P$ on OPEN has an $f$-value, $f(P)$, which lower-bounds the cost of a trajectory that follows $\Gamma$ and initially follows $P$.

Each iteration pops a GCS-path $P$ from OPEN, then iterates over each GCS-node $\mathcal{X}$ adjacent to the last GCS-node in $P$. If $\mathcal{X}$ is a window-node $\mathcal{X}_{i,j}$, and $P$ has not visited the target-windows before $\gamma_{i,j}$ in $\Gamma$, we discard $\mathcal{X}$, since we want a path visiting the target-windows in $\Gamma$ in sequence. If $\mathcal{X}$ is a region-node, and $\mathcal{X}$ already occurs in $P$ after the last window-node in $P$, we discard $\mathcal{X}$, since visiting a region twice without visiting a target in between is suboptimal. If we do not discard $\mathcal{X}$, we create a successor GCS-path $P'$ by appending $\mathcal{X}$ to $P$, then compute $f(P')$ as follows.

Suppose the first target-window in $\Gamma$ unvisited by $P'$ is $\Gamma[n]$. We optimize a trajectory $\taua$ with a collision-free portion $\tau_{\text{a},P'}$ that follows $P'$, an obstacle-unaware portion $\tau_\text{a,next}$ that travels from the end of $\tau_{\text{a},P'}$ to $\Gamma[n]$, and an obstacle-unaware portion $\tau_\text{a,rem}$ that intercepts remaining target-windows in $\Gamma$ after $\Gamma[n]$, in sequence (Fig. \ref{fig:gcs_fig} (a)). For a trajectory $\tau$, let $\dist(\tau)$ be the distance traveled. Naively, we could find $\taua$ that minimizes distance traveled and set $f(P') = \dist(\taua)$. While this provides a lower bound, we use a strengthened bound which considers obstacles after the interception of $\Gamma[n]$.

In particular, our trajectory optimization's minimizes
\begin{align}
    \dist(\tau_{\text{a},P'}) + \dist(\tau_\text{a,next}) + \max\{\dist(\tau_\text{a,rem}), h_n(t)\}\label{eqn:trajopt_objective}
\end{align}
where $t$ is the end time of $\tau_\text{a,next}$, and $h_n$ is a function we design. Specifically, $h_n(t)$ lower-bounds the remaining cost of following $\Gamma$ after intercepting $\Gamma[n]$ at time $t$. We design $h_n$ to be convex to ensure the trajectory optimization is convex.

Let $\gamma_{i,j}$ be $\Gamma[n]$. First, for each segment $\xi_{i,j,k}$ of $\Gamma[n]$, we compute a value $h_{i,j,k}$, which lower-bounds the remaining cost for $t \in [\underline{t}_{i,j,k}, \overline{t}_{i,j,k}]$. In particular, $h_{i,j,k}$ is the cost of the shortest path in the segment-graph for $\Gamma$ from $\xi_{i,j,k}$ to $\xi_{0,2}$, which we compute via the single-source shortest path algorithm for DAGs \cite{cormen2001introduction}. The $h_{i,j,k}$ values can be used to form a piecewise-constant heuristic, but this heuristic would not be convex in general (Fig. \ref{fig:gcs_fig} (b)). Thus, we construct $h_n$ as an affine, and thus convex, function with the following property: for all segments $\xi_{i,j,k}$, for all $t \in [\underline{t}_{i,j,k}, \overline{t}_{i,j,k}]$, $h_n(t)$ lower-bounds $h_{i,j,k}$.


First, if only one segment of $\gamma_{i,j}$ has a finite $h_{i,j,k}$ value, we define $h_n(t) = h_{i,j,k}$. Otherwise, we first construct a temporary function $h_\text{tmp}(t) = mt + b$, by applying ordinary least-squares linear regression to the set of points $\{(\underline{t}_{i,j,k}, h_{i,j,k}) \; | \; k \in \{1, 2, \dots, n_\text{seg}(\gamma_{i,j})\}, h_{i,j,k} \neq \infty\}$. This aims to make $h_\text{tmp}$ roughly match each segment's heuristic value at its start time. However, any finite $m$ and $b$ can be chosen at this step: the next step will still guarantee the lower-bounding property on $h_n$ stated above. Different choices of $m$ and $b$ only affect the bound's tightness.

Next, we define $h_n$ as a line with the same slope as $h_\text{tmp}$, but shifted vertically to ensure that for each segment $\xi_{i,j,k}$, $h_n(t)$ lower-bounds $h_{i,j,k}$ for $t \in [\underline{t}_{i,j,k}, \overline{t}_{i,j,k}]$. First, note that a line lower-bounds a value $h_{i,j,k}$ on interval $[\underline{t}_{i,j,k}, \overline{t}_{i,j,k}]$ if the line lower-bounds $h_{i,j,k}$ at the interval's boundaries. To achieve this for a segment $\xi_{i,j,k}$, we must shift $h_\text{tmp}$ downward by at least $r_k = \max\{h_\text{tmp}(\underline{t}_{i,j,k}) - h_{i,j,k}, h_\text{tmp}(\overline{t}_{i,j,k}) - h_{i,j,k}\}$, i.e. the max overestimation of $h_{i,j,k}$ at the boundaries of the segment's time window. Thus, to get a lower bound for all segments, we must shift $h_\text{tmp}$ down by $r_\text{max} = \max\limits_{k}r_k$. This yields $h_n(t) = h_\text{tmp}(t) - r_\text{max}$.

We use $h_n$ within the objective function \eqref{eqn:trajopt_objective} in our trajectory optimization for GCS-path $P'$. $f(P')$ is the trajectory optimization's optimal cost. The trajectory optimization is convex. Thus, if it is feasible, we find $f(P')$ using Gurobi \cite{gurobi} and push $P'$ onto OPEN. If the optimization is infeasible, we discard $P'$. Once we have found a GCS-path $P^*$ that visits all target-windows in $\Gamma$, such that $f(P') \geq f(P^*)$ for all $P'$ on OPEN, we terminate, and $c^*(\Gamma) = f(P^*)$.
\begin{figure}
    \centering
    \vspace{0.2cm}
    \includegraphics[width=0.47\textwidth]{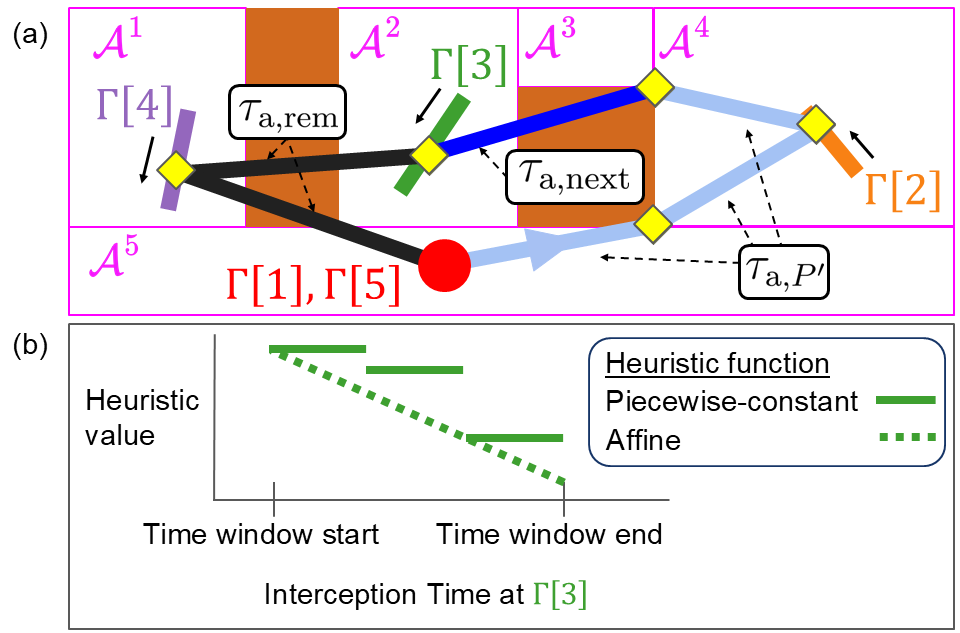}
    \vspace{-0.3cm}
    \caption{(a) Computing $f(P')$ for $P' = (\mathcal{X}_{\Gamma[1]}, \mathcal{X}_{\mathcal{A}^5}, \mathcal{X}_{A^4}, \mathcal{X}_{\Gamma[2]}, \mathcal{X}_{\mathcal{A}^4}. \mathcal{X}_{\mathcal{A}^3})$. Pink boxes are regions decomposing free space. $\tau_{\text{a},P'}$ passes through all sets in $P'$ in sequence. $\tau_\text{a,next}$ travels in a straight line to the next unvisited target-window in $P'$, i.e. $\Gamma[3]$, ignoring obstacles. $\tau_\text{a,rem}$ follows the portion of $\Gamma$ after $\Gamma[3]$, ignoring obstacles. Let $t$ be the time when $\tau_\text{a,next}$ intercepts $\Gamma[3]$. $f(P') = \dist(\tau_{\text{a},P'}) + \dist(\tau_\text{a,next}) + \max\{\dist(\tau_\text{a,rem}), h_3(t)\}$. (b) $h_3(t)$ is an affine function constructed to be no larger than a piecewise-constant heuristic, which we obtain by segmenting the targets' trajectories.}
    \vspace{-0.6cm}
    \label{fig:gcs_fig}
\end{figure}

\section{THEORETICAL ANALYSIS}\label{sec:theoretical_analysis}
\begin{lemma}\label{lemma:final_rmp_cost_lb_mp_cost}
Let $c^*(\LPB)$ be the optimal cost of $\LPB$. If we return no tours in the pricing problem, $\LB(\theta) \leq c^*(\LPB)$.
\end{lemma}
\begin{proof}
Referring to values from Section \ref{sec:branch_and_bound}, strong duality implies
\begin{align}
  \LB(\theta) = \sum\limits_{i = 1}^{\ntar}\lambda_i + \nagt \lambda_0\label{eqn:strong_duality_result}
\end{align}
where the RHS is the cost of $\lambda$ for the dual of the RMP: this dual is the same as (18)-(21) in \cite{feillet2010tutorial}, but costs are replaced with lower bounds. If the labeling algorithm in \cite{bhat2026optimal} returns no tours, it guarantees $\lambda$ is feasible for the dual of $\LPB$ (the same as (18)-(21) in \cite{feillet2010tutorial}, but with a constraint (19) for all tours, not just tours in the restricted subset). The cost of $\lambda$ in the dual of $\LPB$ is $\sum\limits_{i = 1}^{\ntar}\lambda_i + \nagt \lambda_0$, which lower-bounds $c^*(\LPB)$ by weak duality. This and \eqref{eqn:strong_duality_result} imply $\LB(\theta) \leq c^*(\LPB)$.
\end{proof}

\begin{theorem}
Alg. \ref{alg:LazyBPRC} finds an optimal MT-VRP-O solution.
\end{theorem}
\begin{proof}
Let $\mathcal{F}_\text{opt}$ be an optimal MT-VRP-O solution, let $c_\text{opt}$ be its cost, and let $\theta_\text{opt}$ be the corresponding solution to ILP \eqref{optprob:mt_vrp_o_ilp}. We show by induction that whenever we run Alg. \ref{alg:LazyBPRC}, Line \ref{algline:BPRC_main_loop}, either (i) $\UB_\text{inc} = c_\text{opt}$, or (ii) $\theta_\text{opt}$ is feasible for $\LPB$ for some $\mathcal{B}$ on the stack.

\textbf{Base Case} The first node pushed onto the stack is $\mathcal{B} = \emptyset$. $\LPB$ is a relaxation of ILP \eqref{optprob:mt_vrp_o_ilp}, so $\theta_\text{opt}$ is feasible for $\LPB$.

\textbf{Induction Hypothesis} Suppose on Line \ref{algline:BPRC_main_loop}, (i) or (ii) holds.

\textbf{Induction Step} Suppose (i) holds. We never increase $\UB_\text{inc}$, and $\UB_\text{inc}$ cannot be smaller than $c_\text{opt}$ by the optimality of $c_\text{opt}$, so if Line \ref{algline:BPRC_main_loop} is ever executed again, (i) will still hold.

Now suppose (ii) holds. If $\mathcal{B}$ is not popped at this iteration, (ii) still holds. If $\mathcal{B}$ is popped, we have two cases.

\underline{Case 1} Within the Solve-$\LPB$ loop (Line \ref{algline:solve_LPB}), we set $\UB_\text{inc} = c_\text{opt}$. Then (i) holds when we run Line \ref{algline:BPRC_main_loop} next.

\underline{Case 2} $\UB_\text{inc} \neq c_\text{opt}$ after the Solve-$\LPB$ loop. This, and optimality of $c_\text{opt}$, imply $c_\text{opt} < \UB_\text{inc}$. Lemma \ref{lemma:final_rmp_cost_lb_mp_cost} and feasibility of $\theta_\text{opt}$ for $\LPB$ imply $\LB(\theta) \leq c_\text{opt}$. This and $c_\text{opt} < \UB_\text{inc}$ imply $\LB(\theta) < \UB_\text{inc}$. This, and the termination of the Solve-$\LPB$ loop, imply $\theta$ is fractional, as mentioned in Section \ref{sec:branch_and_bound}. This means at least two tours have fractional $\theta_k$. Let $\Gamma$ and $\Gamma'$ be two tours whose $\theta_k$ are fractional, such that $\Length(\Gamma) > 3$.\footnote{Fractional $\theta$ implies such $\Gamma$ exists, if we solve the RMP with a linear program solver that returns basic solutions, e.g. Gurobi \cite{gurobi} with default settings. In particular, consider a solution where each fractionally selected tour has length 3, i.e. it only visits one non-depot target. Consider any non-depot target $i$. In the set of fractionally selected tours visiting $i$, tours only differ in the window where they visit $i$. Tours in this set have the same cost: otherwise, the set's cheapest tour would have $\theta_k = 1$. A basic solution would not fractionally select this set's tours. It would make $\theta_k = 1$ for one tour in the set, make $\theta_k = 0$ for the others, and get the same cost.
}
Let $n$ be the first index such that the $n$th edge  traversed by $\Gamma$ is not incident to the depot, and differs from the $n$th edge traversed by $\Gamma'$. Let $e$ be the $n$th edge of $\Gamma$. $e$ is not disallowed, since $\Gamma$ traverses it. $e$ is not mandated, since $\Gamma'$ leaves its source via a different edge. Thus, some edge is neither disallowed nor mandated nor incident to the depot. Since $\LB(\theta) < \UB_\text{inc}$, the condition on Line \ref{algline:BPRC_check_lb_worse_than_ub} fails and we branch on such an edge. We push successors $\mathcal{B}'$ and $\mathcal{B}''$. If we branch on an edge traversed by $\mathcal{F}_\text{opt}$, $\theta_\text{opt}$ is feasible for $\LPB''$; otherwise, $\theta_\text{opt}$ is feasible for $\LPB'$. Thus (ii) holds when we run Line \ref{algline:BPRC_main_loop} next.

Thus, the induction hypothesis holds the next time we run Line \ref{algline:BPRC_main_loop}. Alg. \ref{alg:LazyBPRC} terminates since the branch-and-bound nodes form a finite tree. Termination only occurs if the stack is empty. Thus at some point, the stack is empty on Line \ref{algline:BPRC_main_loop}, so (ii) cannot hold. Thus (i) holds, i.e. we found $\mathcal{F}_\text{opt}$.
\end{proof}

\section{NUMERICAL RESULTS}\label{sec:numerical_results}
We ran experiments on an Intel i9-9820X 3.3GHz CPU with 10 cores, hyperthreading disabled, and 128 GB RAM. There are no baselines from prior work for the MT-VRP-O, so we compared Lazy BPRC to two ablations. The first ablation, ``Non-Lazy BPRC," is the same algorithm, but when a tour $\Gamma$ is generated in the labeling algorithm (Appendix B) and passes the checks to be added into $\mathcal{F}$, we compute $c^*(\Gamma)$ (if not yet computed), replace $\UB(\Gamma)$ and $\LB(\Gamma)$ with $c^*(\Gamma)$, rerun the checks, and add $\Gamma$ to $\mathcal{F}$ if the checks pass. The second ablation, ``Ignore-Obstacles-Heuristic," replaces our heuristic in the SPP-GCS in Section \ref{sec:tour_cost} with the obstacle-unaware heuristic from \cite{bhat2025AComplete}. That is, in Section \ref{sec:tour_cost}, $\max\{\dist(\tau_\text{a,rem}), h_n(t)\}$ is replaced by $\dist(\tau_\text{a,rem})$. Each algorithm parallelized successor generation in pricing and GCS search, and parallelized computation of segment-to-segment and segment-start-to-segment-start trajectory costs.

We generated problem instances by modifying the instance generation method in \cite{bhat2024AComplete} to handle multiple agents. In all instances, each target had two time windows, demand 1, and speed in each time window generated uniformly at random between 0.5 and 1 m/s. Each instance had three agents with $v_\text{max} = 4$ m/s. Our obstacle maps were square grids, but the agents and targets move in continuous space in the grids. The regions in our GCS were rectangles created by combining unoccupied grid cells as in \cite{feronato2026}. In Experiment 1, we varied the number of targets in random grids, as well as mazes from \cite{sturtevant2012benchmarks}. In random grids, the number of segments per target ($n_\text{seg,tar}$) was set to 64 for Lazy BPRC and Ignore-Obstacles-Heuristic, and 256 for Non-Lazy BPRC. In mazes, $n_\text{seg,tar}$ was 128 for all methods. We tuned these values by varying $n_\text{seg,tar}$ for each method, testing on 10 instances per $n_\text{seg,tar}$, and picking the value giving the best median computation time. We show tuning results in Experiment 2. The computation time limit was 10 min per planner per instance.

\subsection{Experiment 1: Varying Number of Targets}\label{sec:varying_number_of_targets}
We varied $\ntar$ from 6 to 15, setting capacity to $\ntar/\nagt$. We tested on random 30 $\times$ 30 grids, as well as the 10 mazes from the Moving AI benchmark \cite{sturtevant2012benchmarks} with corridor width of 32 cells. Fig. \ref{fig:compare_to_ablations} (a) shows the results. Lazy BPRC's runtime is no larger than the ablations' in any instance, except for one instance in random grids with 6 targets. In this instance, Lazy BPRC was faster than Non-Lazy BPRC and 4\% slower than Ignore-Obstacles-Heuristic. Our contributed GCS heuristic gave little advantage in this instance since GCS search took a minority of the runtime for Ignore-Obstacles-Heuristic. Thus, runtime variance in other algorithmic components (identical for both Lazy BPRC and Ignore-Obstacles-Heuristic) made Lazy BPRC underperform Ignore-Obstacle-Heuristic. To validate this, we ran Lazy BPRC 30 more times on this instance. Lazy BPRC sometimes underperformed and sometimes outperformed Ignore-Obstacles-Heuristic.

Lazy BPRC shows particular advantage over Ignore-Obstacles-Heuristic in mazes. This is expected. In a maze, modifying an obstacle-unaware trajectory to avoid obstacles may require initially traveling in the opposite direction from the original trajectory, to avoid a dead end. In a random grid, avoiding obstacles will often only require perturbing the trajectory at a few points, to avoid blocked cells. Thus, an obstacle-unaware heuristic is expected to perform worse in mazes than in random grids.

To assess the benefit of taking the max of the obstacle-unaware heuristic and our segment-based heuristic in Lazy BPRC, we ran additional experiments where we only used our segment-based heuristic, not taking a max. That is, in Section \ref{sec:tour_cost}, $f(P') = \dist(\tau_{\text{a}, P'}) + \dist(\tau_\text{a,next}) + h_n(t)$. In random grids, this modified Lazy BPRC underperformed Ignore-Obstacles-Heuristic in 8 of 40 instances, with runtime at most 34\% worse. In mazes, the modified Lazy BPRC never underperformed Ignore-Obstacles-Heuristic. Thus, taking the max helps in a minority of cases.

\begin{figure}
    \centering
    \vspace{0.2cm}
    \includegraphics[width=0.47\textwidth]{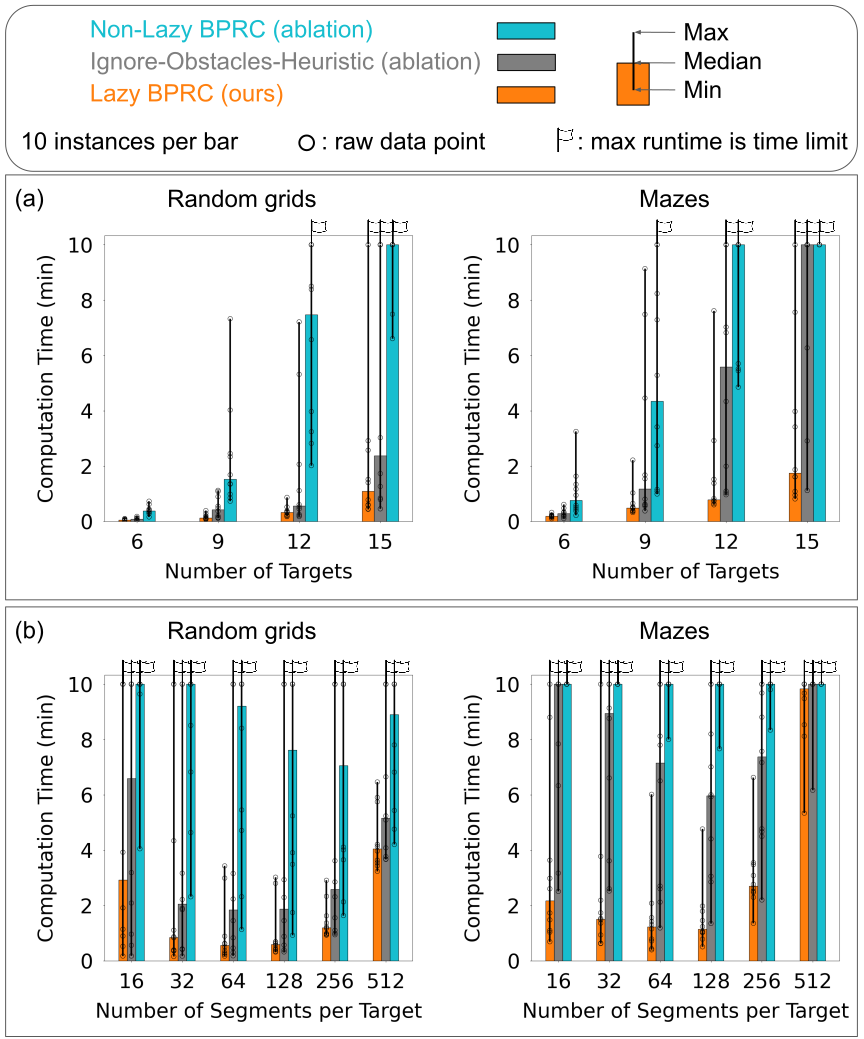}
    \vspace{-0.3cm}
    \caption{(a) Varying number of targets. Lazy BPRC's runtime is less than or equal to the ablations' in all instances, except one instance discussed in Section \ref{sec:varying_number_of_targets}. (b) Varying number of segments per target ($n_\text{seg,tar}$). For each instance-$n_\text{seg,tar}$ pair, Lazy BPRC's runtime is less than or equal to the ablations'.}
    \vspace{-0.6cm}
    \label{fig:compare_to_ablations}
\end{figure}

\subsection{Experiment 2: Varying Number of Segments per Target}\label{sec:varying_number_of_segments}
We fixed $\ntar$ to 12 and capacity to 4, and varied the number of segments per target ($n_\text{seg,tar}$) from 16 to 512. Each $n_\text{seg,tar}$ after 16 is obtained by multiplying the previous $n_\text{seg, tar}$ by 2. As mentioned earlier in the section, the results were used to determine the $n_\text{seg, tar}$ values used in other experiments, so we generated these instances with different random seeds than in Experiment 1. The results are in Fig. \ref{fig:compare_to_ablations} (b). For each instance-$n_\text{seg,tar}$ pair, in both random grids and mazes, Lazy BPRC's runtime is no larger than the ablations'.\footnote{For the instance-$n_\text{seg,tar}$ pair where Ignore-Obstacles-Heuristic spent the least percent of its runtime on GCS search, we saw a similar trend to Experiment 1. That is, if we ran Lazy BPRC multiple times for this pair, sometimes Lazy BPRC underperformed Ignore-Obstacles-Heuristic.} Since Non-Lazy BPRC's median runtime was the time limit for all $n_\text{seg,tar}$ in mazes, we tuned its $n_\text{seg,tar}$ for mazes using the same tuning procedure, but with 9 targets.

\section{CONCLUSIONS}
In this paper, we introduced Lazy BPRC, a new algorithm to find optimal solutions for the MT-VRP-O. One direction for future work is to pursue bounded-suboptimal solutions to enable scaling to more targets.

\section*{ACKNOWLEDGMENT}
This material is partially based on work supported by the National Science Foundation (NSF) under Grant No. 2120219 and 2120529. Any opinions, findings, conclusions, or recommendations expressed in this material are those of the author(s) and do not necessarily reflect views of the NSF.







\bibliographystyle{IEEEtran}
%
\bibliography{refs}

\section*{APPENDIX}\label{sec:appendix}
\subsection{Feasible Solution Generation}\label{sec:feas_soln_gen}
To generate the initial incumbent, as well as a feasible set of tours $\mathcal{F}_\text{new}$ when the RMP is infeasible, we extend the feasible solution generation from \cite{bhat2026optimal} to handle obstacles. In particular, given target-windows $\gamma$ and $\gamma'$, and time $t$, \cite{bhat2026optimal} requires computing, $\EFAT(\gamma, \gamma', t)$, which is the earliest time at which a feasible trajectory can intercept $\gamma'$ after intercepting $\gamma$ at time $t$. $\EFAT$ stands for ``earliest feasible arrival time." \cite{bhat2026optimal} also requires computing $\LFDT(\gamma, \gamma', t)$, which is the latest time that a feasible trajectory can intercept $\gamma$ before intercepting $\gamma'$ at time $t$. $\LFDT$ stands for ``latest feasible departure time." \cite{bhat2026optimal} computes $\EFAT$ and $\LFDT$ in the absence of obstacles. We compute them considering obstacles using the methods from \cite{bhat2024AComplete}.

\subsection{Labeling Algorithm for Pricing Problem}\label{sec:labeling_alg}
To solve the pricing problem in Section \ref{sec:branch_and_bound}, we modify the labeling algorithm in \cite{bhat2026optimal} to handle obstacles, as follows:
\begin{itemize}
    \item \cite{bhat2026optimal} requires computing costs of trajectories between pairs of segments, as well as costs of trajectories between pairs of segment start points, as in Section \ref{sec:lower_bound_on_tour_cost} and \ref{sec:upper_bound_on_tour_cost}. \cite{bhat2026optimal} computes costs in the absence of obstacles. We compute them considering obstacles.
    \item We use obstacle-aware versions of the $\EFAT$ and $\LFDT$ functions, as described in Appendix A.
    \item Let $c_\text{dual}(\Gamma_k) = \sum\limits_{i = 1}^\ntar \alpha(i, \Gamma_k)\lambda_i + \lambda_0$. \cite{bhat2026optimal} only adds a tour $\Gamma$ to $\mathcal{F}$ if both of the following conditions hold:
    \begin{enumerate}[label=(\roman*)]
        \item $c^*(\Gamma) - c_\text{dual}(\Gamma) < 0$
        \item For all previously found tours $\Gamma'$, $c^*(\Gamma) - c_\text{dual}(\Gamma) < c^*(\Gamma') - c_\text{dual}(\Gamma')$.
    \end{enumerate}
    Checking (i) and (ii) requires computing $c^*(\Gamma)$. To avoid this, we instead add $\Gamma$ to $\mathcal{F}$ if $\LB(\Gamma) - c_\text{dual}(\Gamma) < 0$ (necessary for (i)), and $\LB(\Gamma) - c_\text{dual}(\Gamma) < \UB(\Gamma') - c_\text{dual}(\Gamma')$ for all previously found tours $\Gamma'$ (necessary for (ii))). This ensures the tours we add to $\mathcal{F}$ include all tours we would have added using (i) and (ii).
\end{itemize}
For each tour $\Gamma$ generated by the labeling algorithm, $\LB(\Gamma)$ and $\UB(\Gamma)$ are computed as a byproduct: they are $\vec{g}_\text{lb}[1]$ and $\vec{g}_\text{ub}[1]$ in \cite{bhat2026optimal}. The bounds in \cite{bhat2026optimal} did not consider obstacles, but by using our obstacle-aware segment-to-segment and start-to-start trajectory costs in \cite{bhat2026optimal}, we get obstacle-aware bounds. 
\end{document}